\documentclass{article}

\usepackage{microtype}
\usepackage{graphicx}
\usepackage{subfigure}
\usepackage{booktabs} 

\usepackage{hyperref}

\usepackage[accepted]{icml2023}

\usepackage{mathtools}

\usepackage{amsmath,amsfonts,bm}

\def\eqref#1{equation~\ref{#1}}

\def\1{\bm{1}}

\DeclareMathAlphabet{\mathsfit}{\encodingdefault}{\sfdefault}{m}{sl}
\SetMathAlphabet{\mathsfit}{bold}{\encodingdefault}{\sfdefault}{bx}{n}

\usepackage{hyperref}
\usepackage{url}
\usepackage{graphicx}
\usepackage{amsmath}
\usepackage{bbding}
\usepackage{pifont}
\usepackage{wasysym}
\usepackage{amssymb}
\usepackage{amsthm}
\usepackage{bm}
\usepackage{booktabs}
\usepackage{multirow}
\usepackage{caption}
\usepackage{tabularx}
\usepackage{enumitem}
\usepackage{bbm}

\usepackage{comment}

\newtheorem{theorem}{Theorem}

\newtheorem{proposition}{Proposition}

\newtheorem{lemma}{Lemma}

\usepackage{float}
\usepackage{wrapfig,lipsum,booktabs}
\newfloat{figtab}{htb}{fgtb}
\makeatletter
  \newcommand\figcaption{\def\@captype{figure}\caption}
  \newcommand\tabcaption{\def\@captype{table}\caption}
\makeatother

\usepackage{array}
\newcommand{\PreserveBackslash}[1]{\let\temp=\\#1\let\\=\temp}
\newcolumntype{C}[1]{>{\PreserveBackslash\centering}p{#1}}
\newcolumntype{R}[1]{>{\PreserveBackslash\raggedleft}p{#1}}
\newcolumntype{L}[1]{>{\PreserveBackslash\raggedright}p{#1}}

\usepackage[capitalize,noabbrev]{cleveref}
\theoremstyle{plain}
\theoremstyle{definition}

\theoremstyle{remark}

\usepackage[textsize=tiny]{todonotes}

\newcommand{\qa}[1]{{\textcolor{Emerald}{{#1}}}}

\icmltitlerunning{{Bayesian Reprameterized RCRL}}

\begin{document}

\twocolumn[
\icmltitle{{Bayesian Reparameterization of Reward-Conditioned \\ Reinforcement Learning with Energy-based Models}}



\icmlsetsymbol{equal}{*}

\begin{icmlauthorlist}
\icmlauthor{Wenhao Ding}{equal,cmu}
\icmlauthor{Tong Che}{equal,nvidia}
\icmlauthor{Ding Zhao}{cmu}
\icmlauthor{Marco Pavone}{nvidia,stanford}
\end{icmlauthorlist}

\icmlaffiliation{cmu}{Carnegie Mellon University, Pittsburgh, PA, US}
\icmlaffiliation{nvidia}{NVIDIA Research, Santa Clara, CA, US}
\icmlaffiliation{stanford}{Stanford University, Palo Alto, CA, US}

\icmlcorrespondingauthor{Wenhao Ding}{wenhaod@andrew.cmu.edu}
\icmlcorrespondingauthor{Tong Che}{tongc@nvidia.com}

\icmlkeywords{Reward-conditioned reinforcement learning, Energy-based model, Offline reinforcement learning}

\vskip 0.3in
]



\printAffiliationsAndNotice{\icmlEqualContribution} 


\begin{abstract}
Recently, reward-conditioned reinforcement learning (RCRL) has gained popularity due to its simplicity, flexibility, and off-policy nature. However, we will show that current RCRL approaches are fundamentally limited and fail to address two critical challenges of RCRL -- improving generalization on high reward-to-go (RTG) inputs, and avoiding out-of-distribution (OOD) RTG queries during testing time. To address these challenges when training vanilla RCRL architectures, we propose Bayesian Reparameterized RCRL (BR-RCRL), a novel set of inductive biases for RCRL inspired by Bayes' theorem. BR-RCRL removes a core obstacle preventing vanilla RCRL from generalizing on high RTG inputs -- a tendency that the model treats different RTG inputs as independent values, which we term ``RTG Independence". BR-RCRL also allows us to design an accompanying adaptive inference method, which maximizes total returns while avoiding OOD queries that yield unpredictable behaviors in vanilla RCRL methods. We show that BR-RCRL achieves state-of-the-art performance on the Gym-Mujoco and Atari offline RL benchmarks, improving upon vanilla RCRL by up to 11\%. 
\end{abstract}

\section{Introduction}

Reinforcement learning (RL) aims at learning policies to maximize cumulative rewards by trial and error~\cite{sutton1998introduction}. It was shown that when combined with deep neural networks, deep RL is able to learn powerful policies for complex decision-making tasks using only self-generated data~\cite{mnih2013playing, akkaya2019solving, kiran2021deep}. RL algorithms can be divided into two classes in terms of data usage: on-policy RL~\cite{schulman2017proximal, schulman2015trust} algorithms have to be trained on data that the current policy learner generates, while off-policy RL algorithms~\cite{haarnoja2018soft, lillicrap2015continuous} can be trained on data that is generated by some different policies. It is a common belief that off-policy algorithms are more sample efficient than on-policy algorithms~\cite{prudencio2022survey}.

One recent proposed off-policy RL paradigm is reward-conditioned reinforcement learning (RCRL) \cite{kumar2019reward, chen2021decision, janner2021offline, srivastava2019training, ajay2022conditional}, which transforms the RL problem into a conditional sequence modeling problem. The general idea of RCRL is straightforward: we train an RTG (reward-to-go)-conditioned generative model using off-policy data and then set a target RTG when we roll out the policy during training to collect more data~\cite{kumar2019reward} or during testing~\cite{chen2021decision, janner2021offline}. RCRL has gained popularity due to its conceptual simplicity, flexibility, and off-policy nature. Moreover, its usage allows us to leverage powerful neural architectures (e.g., Transformers~\cite{vaswani2017attention}) and large generative models (e.g., diffusion models~\cite{yang2022diffusion} or masked language models~\cite{ghazvininejad2019mask}), which have been a massive success in other parts of artificial intelligence, such as natural language processing~\cite{brown2020language, ouyang2022training} and computer vision~\cite{dosovitskiy2020image, ramesh2022hierarchical}. 

Vanilla RCRL paradigms treat RTGs as standard inputs in addition to the states and actions of the neural network. This design choice makes it easy to employ modern architectures such as Transformers~\cite{vaswani2017attention}. However, this practice, although plausible, has largely ignored the central challenge of RCRL. Training the RCRL model optimizes a policy to fit a dataset or data buffer. However, during policy rollout, we usually want to set a high target RTG in the hope that the model can achieve higher performance than the generating policy of the dataset. In other words, the central challenge of RCRL methods is to achieve generalization from low-return regions to high-return regions of the state space. One core generalization obstacle in vanilla RCRL is that the RTG inputs carry too little information. The model tends to treat different RTGs inputs as independent and unrelated tasks and then fails to find internal connections between trajectories with different RTG levels. In most cases, when we want to learn from sub-optimal datasets or data buffers, high total return trajectories are scarce. Thus, this problem makes learning and generalization in high-return regions extremely difficult. 

In this paper, we propose Bayesian Reparameterized RCRL (BR-RCRL). Our main intuition behind the design is that (1) in order to facilitate generalization from low-return regions to high-return regions, one needs to encode more inductive biases into the model, especially on how to deal with RTGs. (2) one needs to modify the rollout procedure of RCRL to filter out completely out-of-distribution inputs. BR-RCRL is a novel set of inductive biases to parameterize reward-conditioned policies inspired by Bayes' theorem. In BR-RCRL, we explicitly impose the prior knowledge that different RTGs are competitive with each other, not independent. From a causality perspective, BR-RCRL can also be viewed as a \textit{causal generative model}~\cite{peters2017elements} that respects the ground-truth causal relationships between random variables. This causal viewpoint provides a concrete explanation of why BR-RCRL generalizes better than vanilla RCRL models when distribution shifts occur.

In the experiment section, we show that BR-RCRL dramatically boosts the performance of many RCRL algorithms across commonly used offline RL benchmarks.

\section{Related Works and Preliminaries}

\subsection{Off-policy and Offline RL}
Reinforcement learning studies learning problems in the setting of a Markov Decision Process (MDP) described by the tuple ($\mathcal{S}$, $\mathcal{A}$, $P$, $\mathcal{R}$). This tuple consists of action $a \in \mathcal{A}$, state $s \in \mathcal{S}$, transition probability $P(s'|s, a)$, and reward function $r=\mathcal{R}(s, a)$. The goal of MDP is to maximize the expected return $\mathbb{E}[\sum_{t=0}^T \gamma^t r_t]$, where we denote $\gamma$ the discount factor and $a_t$, $s_t$, $r_t=\mathcal{R}(s_t, a_t)$ the action, state, and reward at timestep $t$, respectively.

In the on-policy RL setting~\cite{schulman2017proximal, schulman2015trust}, an agent interacts with the environment and updates its current policy using experiences gathered using the same current policy. In off-policy RL~\cite{mnih2013playing, lillicrap2015continuous}, the agent still interacts with the environment, but can update its current policy using experiences collected from any past policies as well.
The off-policy framework brings us two advantages:
(1) More sample-efficient training since the agent does not have to discard all previous transitions and can instead maintain a buffer where transitions can be reused multiple times.
(2) Better state space exploration since the sample collection follows a behavior policy different from the target policy.

Offline RL~\cite{prudencio2022survey, levine2020offline}, also known as Batch RL, moves one step further and becomes truly ``off-policy'' by learning only from static dataset $\{s_i,a_i,s'_i, r_i\}_{i=1}^N$ collected from arbitrary policies. 
This offline setting can be extremely valuable when an online interaction is impractical due to expensive or dangerous data collection such as robotics~\cite{singh2021reinforcement}, healthcare~\cite{liu2020reinforcement}, and autonomous driving~\cite{kiran2021deep}. 
However, the main challenge in offline RL is the distributional shift between the dataset and the environment. This challenge is either addressed by constraining the learned policy to the behavior policy used to collect the dataset~\cite{fujimoto2019off, kumar2019stabilizing, wu2019behavior} or estimating a conservative value function~\cite{kumar2020conservative, yu2021combo}.

\subsection{Reward-Conditioned RL}
A popular off-policy learning paradigm is Reward-Conditioned Reinforcement Learning, which has been studied across multiple contexts~\cite{janner2021offline, chen2021decision, kumar2019reward, emmons2021rvs, srivastava2019training, ajay2022conditional}.
We denote the data-generating behavior policy to be $\beta$, and then we define the random variable of total return (RTG) after taking action $a$ at state $s$ as
\begin{equation}
    Z^\beta(s, a)=\sum_{t\geq 0}\gamma^t r(a_t,s_t)|_{a_0=a, s_0=s}.
\end{equation}
RCRL tries to learn the RTG-conditioned policy $\bar{\beta}_\theta(a|R,s)$ parameterized by $\theta$ to match the ground-truth posterior policy $\beta[a|s, Z^\beta(s, a)=R]$, where $R$ is a target RTG. In vanilla RCRL settings, the RTGs are treated as an input variable and directly fed into a neural network, which could be an MLP~\cite{emmons2021rvs}, a Transformer~\cite{chen2021decision, janner2021offline}, or a diffusion model~\cite{ajay2022conditional}. 
Namely, the learned policy $\bar{\beta}_\theta(a|R,s)$ takes $R$ and current state $s$ and outputs a distribution of actions that matches the posterior distribution of data generating policy $\beta(a|s, R)$. 
More formally, given a data buffer $\{s_i,a_i,s'_i,R_i\}_{i=1}^N$, RCRL optimizes the following loss function:
\begin{equation}
    \mathcal{L}_{\text{RCRL}}(\theta) = -\sum_i \log \bar{\beta}_\theta(a_i|s_i, R_i)
\end{equation}
In an online RCRL setting~\cite{kumar2019reward}, one interleaves the policy fitting with data collection and dynamically expands the dataset with new data collected. 

\subsection{Energy-based Models}
Energy-based Models (EBMs)~\cite{lecun2006tutorial, song2021train}, also known as non-normalized probabilistic models, are flexible and can model expressive distributions since they do not have a restriction on the tractability of the normalizing constant. The density given by an EBM is $p_{\theta}(\bm{x})=\exp(-E_{\theta}(\bm{x}))/Z_{\theta}$, where energy $E_{\theta}(x)$ is a nonlinear function parameterized by $\theta$ and $Z_{\theta}=\int \exp(-E_{\theta}(\bm{x})) d\bm{x}$ is a constant w.r.t $\bm{x}$. 
Although the likelihood of EBMs cannot be directly maximized, three surrogate principles for learning EBMs are usually considered. Firstly, we can estimate the gradient of the log-likelihood with MCMC approaches~\cite{neal2011mcmc, welling2011bayesian}, which use the fact that the gradient of the log-probability w.r.t. $\bm{x}$ equals the gradient of the energy. 
Secondly, one can learn an EBM by matching the first derivatives of the density function and the data distribution~\cite{song2019generative, song2020score}. Finally, an EBM can be learned by contrasting it with another distribution with a known density using Noise Contrastive Estimation (NCE)~\cite{gutmann2010noise}. InfoNCE~\cite{oord2018representation}, inspired by NCE, uses categorical cross-entropy loss to identify the positive sample $\bm{x}$ amongst a set of unrelated noise samples $X'$.
\begin{equation}
    \mathcal{L}_{\text{infoNCE}} = -\mathbb{E}\left[\log\frac{f(\bm{x}, \bm{c})}{\sum_{\bm{x}'\in X'} f(\bm{x}', \bm{c})} \right]
\end{equation}
where $\bm{c}$ is the context indicating the label of $\bm{x}$ and the scoring function is $f(\bm{x}, \bm{c}) \propto \frac{p(\bm{x}|\bm{c})}{p(\bm{x})}$.

Recently, EBM formulation has been considered the policy representation~\cite{haarnoja2017reinforcement}. Existing works also use EBMs in a model-based planning framework~\cite{boney2020regularizing} or imitation learning~\cite{liu2020energy} with an on-policy algorithm.
Another trend to combine EBM and RL is utilizing an EBM as part of the RL framework~\cite{kostrikov2021offline, nachum2021provable}.

\section{Limitations of Vanilla RCRL}

The pipeline of vanilla RCRL has two fundamental limitations. Although the loss $\mathcal{L}_{\text{RCRL}}$ is a reasonable surrogate of $-\mathbb{E}_{\beta}[\log \bar{\beta}(a|s, R)]$, optimizing such a loss is equivalent to maximizing the average likelihood under the distribution induced by $\beta$. However, we aim to achieve a higher reward than we can get from $\beta$ during test time. The mismatch between the training and testing input distribution becomes more severe when we try to set a higher RTG for the model. In many cases, the input RTGs can become so high that they turn out to be out-of-distribution inputs (c.f. Figure ~\ref{fig:trajectory}).  Thus, we argue that a fundamental challenge for RCRL is to improve the model's generalization performance on higher RTG inputs while avoiding unpredictable behavior caused by OOD inputs. 

However, vanilla RCRL models lack appropriate inductive biases to facilitate such generalization. One core issue that makes the generalization to high RTG region extremely difficult is that the model could treat inputs with specific RTGs as independent, unrelated prediction problems and fail to discover the connections across trajectories with different RTGs (for which we term as \textbf{RTG Independence}). Since RTGs contain minimal information, in our experiments, we found that in contrast with what one may hope, many vanilla RCRL models~\cite{emmons2021rvs, chen2021decision} have a tendency to treat each different RTG input simply as one independent prediction task. This phenomenon prohibits the model from learning useful information in low RTG regions and tries to generalize to high RTG regions. (The generalization is notoriously hard because, in almost all RL tasks, low-reward trajectories look very different from high-reward ones.)

\begin{figure}[t]
    \centering
    \includegraphics[width=0.49\textwidth]{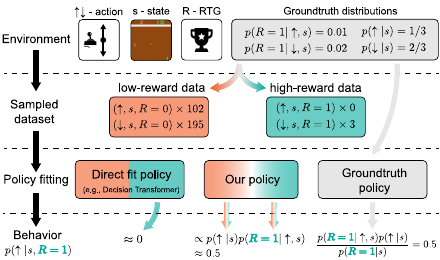}
    \caption{Illustration of the \textbf{Sampling Bias Dominance} problem. When conditioned on $R=1$, only 3 noisy samples are available; thus, a huge sampling variance is introduced. Our policy uses an amortized estimation of $p(a|s)$ and $p(R|a,s)$, so it is less affected by the high sampling variance. }
    \label{fig:example}
    \vspace{-5mm}
\end{figure}

In fact, these issues result in two serious problems for RCRL during both training and testing. The first problem is \textbf{Sampling Bias Dominance}. When conditioned on higher RTGs, the training samples that can attain these high RTGs become fewer. This means that empirical distribution in the dataset, when conditioned on a high RTG, can look very different from the ground-truth data-generating distribution because of the large sampling variance. The sampling variance, when combined with the RTG Independence problem, makes the generalization on high RTG regions notoriously difficult. Consider a simple example environment in Figure~\ref{fig:example}, where we have 2 possible RTG values for a given action. The RTG input tokens carry only one bit of information, and it does not provide any information or hints about how the target RTGs $R=1$ and $R=0$ are related to each other. The model has no choice but to view $p(a|R=1,s)$ and $p(a|R=0,s)$ as independent and unrelated prediction problems. The model has no idea about simple facts such as these two options $R=0$ and $R=1$ should be mutually exclusive. What makes things worse, high RTG samples are rare. In this case, $R=1$ consists of only 3 samples. Thus, the empirical distribution $p_e(\uparrow|R=1,s)=0$ based on the dataset is quite different from the ground-truth distribution $p(\uparrow|R=1,s)=0.5$. The prediction problem $f: S\rightarrow \Delta(\mathcal{A})$ where $f(s)=p(a|R=1,s)$ is thus too noisy to be trained on this dataset due to sampling bias. Therefore, Vanilla RCRL methods have difficulties in generalization to high reward regions because of such large sampling noises in the dataset. 

The second problem, which mainly occurs in testing time, is what we term as \textbf{Out-of-Distribution (OOD) Conditioning}. We provide a stair-climbing example in Figure~\ref{fig:trajectory} for illustration. During testing time, a typical practice is first to choose a high target RTG and then dynamically decrease the RTG according to immediate rewards from the environment. In this example, if the robot takes one bad move in $s_2$, then quickly the conditioning RTG would not be achievable from the next state $s_3$ (because the goal can never be achieved with 1 remaining step), which means the RTG condition turns into an OOD input to the policy model. The behavior of the trained neural network is undefined on such inputs, bringing in severe performance drops. 

\begin{figure}[t]
    \centering
    \includegraphics[width=0.48\textwidth]{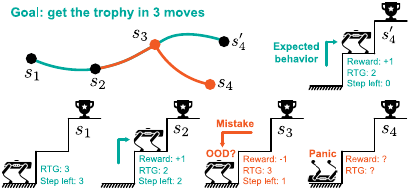}
    \caption{Illustration of the \textbf{OOD Conditioning} problem, where an agent accidentally takes a bad move and encounters an OOD, it then gets stuck there and is ``panic''. Our algorithm can dynamically select the RTG to maximize the return when such bad things happen.}
    \label{fig:trajectory}
\end{figure}

\section{Methodology}

\subsection{Bayesian Reparameterization of RCRL}

In our method, the goal is to learn a probabilistic model to better approximate the posterior policy $\beta[a|s, Z^\beta(s, a)=R]$. Instead of directly taking RTGs as inputs to the neural network, our observation is that one could encode more prior knowledge into the model. A core inductive bias we introduce is that different RTGs should be competitive, not independent of each other. In order to encode this competition between different RTGs into the model, we no longer feed RTGs into the model as an extra input variable. Instead, the RTG mechanism is replaced by an energy function defined by two amortized neural network estimations of $\beta(a|s)$ and $\beta(R|s, a)$.

More precisely, thanks to the Bayes formula, the posterior policy can be written as: 
\begin{equation}
    \beta[a|s,Z^\beta(s,a)=R]\propto \beta(a|s)\beta[Z^\beta(s,a)=R|s,a].
\end{equation}
Inspired by the Bayesian representation of the posterior policy, we define an energy-based model
\begin{equation}
      E_\theta(a|s, R) = -\log \bar{\beta}_\theta(a|s)-\log \bar{\beta}_\theta(R|s,a),
\end{equation}
in which $\bar{\beta}_{\theta}(\cdot|s)$ and $\bar{\beta}_{\theta}(\cdot|s,a)$ are represented as two parameterized neural networks. In this work, we further assume that $\bar{\beta}_{\theta}(a|s)$ is easy to sample from, and its likelihood can be computed exactly. Instead of directly fitting a conditional model like vanilla RCRL, we reparameterize RCRL with the policy defined by this energy function $\bar{\beta}_\theta(a|s, R) = \exp(-E_\theta(a|s, R))/Z$ as an approximation to real $\beta(a|s, R)$, where $Z_{\theta}=\sum_a \exp(-E_{\theta}(a|s, R))$. 

In vanilla RCRL, we optimize the following loss function:
\begin{equation}
    \mathcal{L}_0(\theta) = -\sum_i \log \bar{\beta}_\theta(a_i|s_i, R_i),
\label{equ:rcrl}
\end{equation}
where $\bar{\beta}_\theta(a|s, R)$ is parameterized as a normal neural network with input $s, R$ and output $a$. In our model, we still optimize the above loss function but with a Bayesian way of parameterization:
\begin{equation}
    \bar{\beta}_{\theta}(a|s,R) = \frac{\exp(-E_{\theta}(a|s,R))}{Z_{\theta}}=\frac{\bar{\beta}_{\theta}(a|s)\bar{\beta}_{\theta}(R|s,a)}{Z_{\theta}}
\end{equation}
In order to learn such a model, we need to calculate the normalizing constant $Z$. Here we use samples from $\bar{\beta}_{\theta}(a|s)$ to estimate it:
\begin{equation}
    Z_{\theta}=\sum_a \bar{\beta}_{\theta}(a|s)\bar{\beta}_{\theta}(R|s,a)=\mathbb{E}_{a\sim \bar{\beta}_{\theta}}[\bar{\beta}_{\theta}(R|s,a)].
\end{equation}

We then use InfoNCE~\cite{oord2018representation} loss to optimize objective (\ref{equ:rcrl}) after rewriting the model as
\begin{equation}
\bar{\beta}_{\theta}(a|s, R, A') = \frac{\bar{\beta}_{\theta}(a|s)\bar{\beta}_{\theta}(R|s,a)}{\sum_{a'\in A'}\bar{\beta}_{\theta}(R|s,a')},
\end{equation}
where negative samples $a' \in A'$. Thus, our loss function for RCRL can be summarized as
\begin{equation}
\begin{split}
      \mathcal{L}_0(\theta) 
    =& -\sum_i \left[ \log \bar{\beta}_{\theta}(a_i|s_i)+ \log \frac{\bar{\beta}_{\theta}(R_i|s_i,a_i)}{\sum_{a'_i\in A_i'}\bar{\beta}_{\theta}(R_i|s_i,a'_i)} \right],
\end{split}
\end{equation}
where $a_i'$ is sampled from $\bar{\beta}_{\theta}(a|s_i)$. In addition to the $\mathcal{L}_0(\theta)$ loss that has the same goal as vanilla RCRL, we add term $\mathcal{L}_1(\theta)$, which is the log-likelihood of the RTG:
\begin{equation}
    \mathcal{L}_1(\theta)=-\sum_i \log \bar{\beta}_\theta(R_i|s_i,a_i).
\end{equation}
This term is essential for our model since it addresses the RTG independence problem. After having this loss, the model is forced to capture the critical prior knowledge of how different RTG inputs depend on each other.  
Finally, our objective is to maximize the combination of the two losses with an adjustable parameter $\lambda$:
\begin{equation}
    \mathcal{L}_{\text{BR-RCRL}}(\theta) = \mathcal{L}_0(\theta)+\lambda \mathcal{L}_1(\theta).
\end{equation}

\subsection{Adaptive Inference}

In order to address the OOD Conditioning problem, we propose a novel adaptive inference method to ensure that our query is always in training distribution. After training $\bar{\beta}_{\theta}(a|s, R)$, we aim to deduce a new policy that can perform better than the data-generating policy $\beta$. To do so, we write $Z^{\bar{\beta}}(s)=\sum_{t\geq 0}\gamma^t r(a_t,s_t)|_{s_0=s}$ the expected total return under $\bar{\beta}$. For a given $\delta\in(0,1)$, we define a threshold function 
\begin{equation}
\theta_\delta(s)= \max_r \{r\in \mathbb{R}|P(Z^{\bar{\beta}}(s)\geq r)\geq \delta\}.
\end{equation}
Then, we define the new policy as 
\begin{equation}
    \pi^\delta(a|s)=\bar{\beta}(a|Z^{\bar{\beta}}(s,a)\geq \theta_\delta(s) ,s).
\end{equation}
During testing time, we sample from $\pi^\delta(a|s)$. Intuitively, this sampling method tries to dynamically adjust the RTG as high as possible while preserving feasibility.

By definition we have $\bar{\beta}_\theta(a|s,R) = \exp(-E_\theta(a|s,R))/Z$. Combing these formulas, we have the following proposition:
\begin{proposition}
   Define a new energy function :  $E^\delta_\theta(a|s) = -\log \bar{\beta}_\theta(a|s)-\log \bar{\beta}_\theta(R>\theta_\delta(s)|s,a)$, where 
   \begin{equation}
        \bar{\beta}_\theta(R>\theta_\delta(s)|s,a) =\sum_{r>\theta_\delta(s)}  \bar{\beta}_\theta(r|s,a) 
   \end{equation}
then we have  $\pi^\delta(a|s)\propto E^\delta_\theta(a|s)$. Namely, $\pi^\delta(a|s)$ is the Boltzmann distribution of energy function $E^\delta_\theta(a|s)$. 
\label{the:1}
\end{proposition}
\begin{proof}
It directly follows from rewriting $\bar{\beta}(a|Z^{\bar{\beta}_{\theta}}(s,a)\geq \theta_\delta(s) ,s)$ with Bayes formula.
\end{proof}

In our implementation, $\theta_\delta(s)$ is not exactly computable, so we approximate it using samples. We first sample a batch of actions $\{a_i\}_{i=1}^N\sim \bar{\beta}_{\theta}(a|s)$, then use these samples to compute an estimate of the distribution $\bar{\beta}_{\theta}(R|s)=\frac{1}{N}\sum_i[\bar{\beta}_{\theta}(R|s,a_i)]$. Then we use the threshold $\delta$ to select the threshold $\theta_\delta$ such that 
\begin{equation}
\theta_\delta(s)= \max_r \{r\in \mathbb{R}|P(R\geq r|s)\geq \delta\}
\end{equation}
In summary, we illustrate the proposed adaptive inference procedure in Algorithm~\ref{alg:inference}.

\begin{algorithm}[tb]
   \caption{Adaptive Inference for BR-RCRL}
   \label{alg:inference}
\begin{algorithmic}[1]
   \STATE {\bfseries Input:} threshold $\delta>0$, trained policy network $\bar{\beta}_\theta$
   \STATE Initialize $s=s_0$.
   \WHILE{$s$ is not a terminal state}
   \STATE Sample a batch of actions $\{a_i\}_{i=1}^N\sim \bar{\beta}_{\theta}(a|s)$. 
   \STATE Compute $\bar{\beta}_{\theta}(R|s)=\frac{1}{N}\sum[\bar{\beta}_{\theta}(R|s,a_i)]$
   \STATE Compute $\theta_\delta(s)= \max_r \{r\in \mathbb{R}|P(R\geq r|s)\geq \delta\}$
   \STATE Perform iterative inference on energy function $E^\delta_\theta(a|s)$ and return the best action $a'$
   \STATE Execute $a'$ in the environment and get $s'$
   \STATE Update $s' \rightarrow s$
   \ENDWHILE
\end{algorithmic}
\end{algorithm}

\subsection{Discrete Action}
In the discrete action case, our model can be considerably simplified. 
Following~\cite{bellemare2017distributional}, we model the distribution of RTG using a discrete distribution parametrized by $V_{\text{min}} \in \mathbb{R}$, $V_{\text{max}} \in \mathbb{R} $, and $N \in \mathbb{N}$.
Specifically, we use a set of bucket $B$:
\begin{equation}
    \left\{b_i=V_{\text{min}} + i\Delta b: i\in [0, N) \right\}, \Delta b = \frac{V_{\text{max}}-V_{\text{min}}}{N-1},
\end{equation}
to represent the discretized RTG.
In this case, we can train a joint probabilistic model $\bar{\beta}_{\theta}(a, R|s)$, which is parameterized as $f: S\rightarrow \Delta(\mathcal{A} \times \mathbb{R})$. The set of joint distributions $\Delta(\mathcal{A}\times \mathbb{R})$ is represented using a $|A||B|$-way Softmax function~\cite{bridle1989training}, where $|A|$ is the number of actions and $|B|$ is the number of reward buckets.

In this parameterization, the $\mathcal{L}_1(\theta)$ loss stays the same but the InfoNCE~\cite{oord2018representation} loss can be replaced by a normal Softmax loss $\mathcal{L}_0(\theta) = -\sum_i \log \bar{\beta}_\theta(a_i|s_i, R_i)$ because now the conditional distribution $\bar{\beta}_{\theta}(a|s, R)$ can be easily computed.

\begin{figure}[t]
    \centering
    \includegraphics[width=0.4\textwidth]{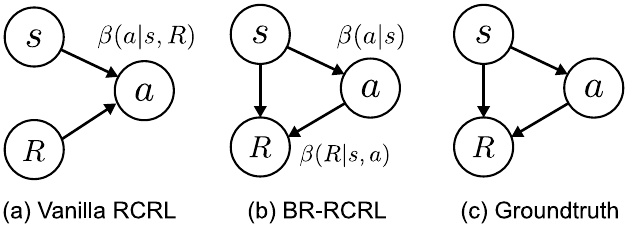}
    \caption{Generative models of Vanilla RCRL (a), BR-RCRL (b), and the ground-truth causal graphical model of RTG generation (c).}
    \label{fig:causal_pgm}
\end{figure}

\subsection{Causal Perspective of BR-RCRL}
We provide another explanation of why BR-RCRL can facilitate better generalization over vanilla RCRL. Consider the ground-truth causality relationships between three random variables $(s, a, R)$ generated by the data-generating policy. It is obvious that $s$ is a direct cause of $a$ since $a$ is generated by the behavior policy $\beta$. It is also apparent that both $s$ and $a$ are immediate causes of the total return $R$, leading to the true causal graph as shown in Figure~\ref{fig:causal_pgm}(c). BR-RCRL tries to fit two neural networks $\bar{\beta}_\theta(a|s)$ and $\bar{\beta}_\theta(R|s, a)$ that are aligned with the true causal model. It is well-known that generative models that respect the causality relationships are more robust to distribution shifts because they can avoid learning spurious relationships between random variables~\cite{arjovsky2019invariant, scholkopf2021toward, lu2021invariant, ding2022generalizing}. This explains why our BR-RCRL generalizes much better than vanilla RCRL when the distribution shifts because of the user-specified value of $R$, which can be viewed as an intervention~\cite{eberhardt2007interventions} from a causality point of view. 

\subsection{Theoretical Analysis}

Now we provide an analysis of the proposed algorithm. 
In the following, we sometimes write $\pi=\pi^\delta$ for short if there is no confusion.
Theoretically, we expect our new policy $\pi^\delta$ satisfy two important properties: 
\begin{itemize}[leftmargin=0.3in]
    \vspace{-2mm}
    \item[(1)] Its trajectory distribution should not diverge from $\beta$, which means the model should not take OOD actions. This can be quantitatively measured by KL divergence $\text{KL}[\pi^\delta(\mathcal{T})||\beta(\mathcal{T})]$. If the KL divergence is bounded, each sample trajectory in $\pi^\delta$ has a positive probability in $\beta$, thus eliminating the OOD problem.
    \vspace{-3mm}
    \item[(2)] Policy $\pi^\delta$ should be guaranteed to have better performance than $\beta$. Otherwise, one can do behavior cloning and ignore the reward information. 
    \vspace{-2mm}
\end{itemize}
In summary, we have the following theorem to justify our algorithm. The proof can be found in Appendix~\ref{app:proof}.
\begin{theorem}
    Let $\delta\in (0,1)$, $\beta$ be the data generating policy and $\pi^\delta$ be the conditional policy $\pi^\delta(a|s)=\beta(a|Z(s,a)\geq \theta_\delta(s) ,s)$. Then we have $\text{KL}[\pi^\delta(\mathcal{T})||\beta(\mathcal{T})] \leq -N\log \delta$. On the other hand, we have $V^{\beta}(s)\leq V^{\pi}(s), \forall s\in S$. $V^\beta(s) = \mathbb{E}_{a\sim \beta,\beta}[Z^\beta(s,a)]$, and $V^\pi(s) = \mathbb{E}_{a\sim \pi,\pi}[Z^\pi(s,a)]$.
\label{the:2}
\end{theorem}

\begin{table*}[t]
\centering
\caption{Normalized Scores on Gym-MuJoCo tasks. The results of our method are averaged over 5 random seeds.}
\small{
  \begin{tabular}{l|l|C{1.0cm}C{0.6cm}C{0.6cm}C{0.6cm}C{0.6cm}C{0.5cm}C{0.7cm}C{0.5cm}C{0.8cm}C{0.6cm}C{0.6cm}C{0.7cm}}
    \toprule    
    \textbf{Dataset} & \textbf{Environment} & \textbf{Ours} & \textbf{DD} & \textbf{TT} & \textbf{DT} & \textbf{RvS} & \textbf{BC} & \textbf{\scriptsize{10\%}}\textbf{BC} & \textbf{IBC} & \textbf{TD3\scriptsize{+BC}} & \textbf{IQL} & \textbf{CQL} & \textbf{BEAR} \\
    \midrule
    Med-Expert & HalfCheetah &  \textbf{95.2}\scriptsize{$\pm$0.8} &  90.6 &  {95.0} &  86.8 &  92.2 &  55.2 &  92.9 &  34.8 &  90.7 &  86.7 &  91.6 &  53.4 \\
    Med-Expert & Hopper &  \textbf{112.9}\scriptsize{$\pm$0.9} &  {111.8} &  {110.0} &  107.6 &  101.7 &  52.5 &  {110.9} &  27.5 &  98.0 &  91.5 &  105.4 &  96.3 \\
    Med-Expert & Walker2d &  \textbf{111.0}\scriptsize{$\pm$0.4} &  108.8 &  101.9 &  108.1 &  106.0 &  107.5 &  109.0 &  16.2 &  {110.1} &  109.6 &  108.8 &  40.1 \\
    \midrule
    Medium & HalfCheetah &  {48.6}\scriptsize{$\pm$1.1} &  \textbf{49.1} &  46.9 &  42.6 &  41.6 &  42.6 &  42.5 &  35.2 &  48.3 &  47.4 &  44.0 &  41.7 \\
    Medium & Hopper &  {78.0}\scriptsize{$\pm$1.3} &  \textbf{79.3} &  61.1 &  67.6 &  60.2 &  52.9 &  56.9 &  75.3 &  59.3 &  66.3 &  58.5 &  52.1 \\
    Medium & Walker2d &  {82.3}\scriptsize{$\pm$1.7} &  {82.5} &  79.0 &  74.0 &  71.7 &  75.3 &  75.0 &  14.7 &  \textbf{83.7} &  78.3 &  72.5 &  59.1 \\
    \midrule
    Med-Replay & HalfCheetah &  42.3\scriptsize{$\pm$3.3} &  39.3 &  41.9 &  36.6 &  38.0 &  36.6 &  40.6 &  24.5 &  {44.6}&  {44.2} &  \textbf{45.5} &  38.6 \\
    Med-Replay & Hopper &  {98.3}\scriptsize{$\pm$2.6} &  \textbf{100} &  91.5 &  82.7 &  73.5 &  18.1 &  75.9 &  12.4 &  60.9 &  94.7 &  95.0 &  33.7 \\
    Med-Replay & Walker2d &  {80.6}\scriptsize{$\pm$2.5} &  75 &  \textbf{82.6} &  66.6 &  60.0 &  26.0 &  62.5 &  9.4 &  {81.8} &  73.9 &  77.2 &  19.2 \\
    \midrule
    \multicolumn{2}{c|}{\textbf{Average Score}}  & \textbf{83.2} & 81.8 & 78.9 & 74.7 & 71.7 & 51.9 & 74.0 & 27.8 & 75.3 & 77.0 & 77.6 & 48.2 \\
    \bottomrule
  \end{tabular}
}
\label{tab:d4rl-results}
\end{table*}


\section{Experiment}

In this section, we conduct several experiments on two standard benchmarks to answer the following questions:
\begin{itemize}[leftmargin=0.2in]
    \vspace{-3mm}
    \item \qa{\textbf{Q1:}} How is the performance of our proposed method compared to existing offline RL methods?
    \vspace{-3mm}
    \item \qa{\textbf{Q2:}} How do different target RTG strategies during inference influence the results?
    \vspace{-3mm}
    \item \qa{\textbf{Q3:}} How does the observed RTG match the target RTG during the inference stage?
    \vspace{-3mm}
    \item \qa{\textbf{Q4:}} How do different components in BR-RCRL influence the performance?
    \vspace{-3mm}
\end{itemize}
We first briefly introduce the datasets and settings used in the experiment, then provide answers to the above questions and additional analyses.

\subsection{Benchmarks and Datasets}

We evaluate our method in 9 Gym-MuJoCo tasks~\cite{fu2020d4rl} and 4 Atari games~\cite{mnih2013playing}, which are both standard offline RL benchmarks and cover continuous and discrete action spaces.
Results of baselines are obtained from the original papers except for DT on Atari because the reported score is obtained with the 1\% buffer dataset.

Datasets of the Gym-MuJoCo tasks are collected in locomotion environments (HalfCheetah, Hopper, and Walker2D) with three different data buffers.
{\fontfamily{pcr}\selectfont Medium} is generated by first training an online Soft Actor-Critic (SAC)~\cite{haarnoja2018soft} model, early-stopping the training, and collecting 1 million samples from this partially-trained policy.
{\fontfamily{pcr}\selectfont Medium-Replay} consists of recording all samples in the replay buffer observed during training until the policy reaches the ``medium'' level of performance. 
{\fontfamily{pcr}\selectfont Medium-Expert} mixes one million expert demonstrations and one million suboptimal data generated by a partially trained policy or by unrolling a uniform-at-random policy.
The results are normalized to ensure that the well-trained SAC model has a 100 score and the random policy has a 0 score.

The Atari benchmark is more difficult due to the high-dimensional state space and the long-horizon delayed reward. The offline dataset of this benchmark is collected from the replay buffer of an online DQN agent~\cite{mnih2015human}. The entire dataset has 50 million transitions, but we follow the setting in~\cite{kumar2020conservative} and use 10\% of the buffer (5 million transitions). Following~\cite{hafner2020mastering}, we report the normalized score where the random policy is 0 and the human performance is 100.

\subsection{Overall Performance (\qa{Q1})}

The overall performance in Gym-MuJoCo and Atari benchmarks is reported in Table~\ref{tab:d4rl-results} and Table~\ref{tab:atari-results} with the comparison of three types of baselines:
\begin{itemize}[leftmargin=0.2in]
    \vspace{-2mm}
    \item \textbf{Offline TD learning} adds constraints to online RL methods that use TD error, leading to a pessimistic behavior policy or a conservative value function. In this paper, we compare with QR-DQN~\cite{dabney2018distributional}, REM~\cite{agarwal2020optimistic}, IQL~\cite{kostrikov2021offline}, CQL~\cite{kumar2020conservative}, and BEAR~\cite{kumar2019stabilizing}.
    \vspace{-2mm}
    \item \textbf{Reward-conditioned RL} takes state and RTG (or reward) as input and predicts actions for the next step. We consider four representative works in the experiment: Trajectory Transformer (TT)~\cite{janner2021offline}, Decision Transformer (DT)~\cite{chen2021decision}, Decision Diffuser (DD)~\cite{ajay2022conditional}, and RvS~\cite{emmons2021rvs}.
    \vspace{-2mm}
    \item \textbf{Imitation learning} uses supervised learning to train policy, which mimics the state-action pairs in the dataset and usually ignores the reward information. We consider Behavior Cloning (BC)~\cite{pomerleau1988alvinn}, BC with top 10\% data (10\%BC), Implicit BC (IBC)~\cite{florence2022implicit}, and TD3+BC~\cite{fujimoto2021minimalist} as our baselines.
    \vspace{-2mm}
\end{itemize}

According to Table~\ref{tab:d4rl-results}, our method outperforms all baselines in terms of the average score. We achieve the highest score in 3 out of 9 Gym-MuJoCo tasks and are very close to the highest score (within standard derivation) in the remaining tasks. 
Unsurprisingly, our method works well in sub-optimal datasets (i.e., {\fontfamily{pcr}\selectfont Medium} and {\fontfamily{pcr}\selectfont Medium-Replay}) since the Bayesian reparametrization generalizes well even though only trained with low-reward data.
Compared to strong architectures, i.e., Transformers (DT and TT) and diffusion models (DD), our method still achieves improvement in most dataset settings, demonstrating the advantages of generalizability.

In the Atari benchmark, as shown in Table~\ref{tab:atari-results}, our method has the highest score in 3 out of 4 games and achieves a significant improvement in the {\fontfamily{pcr}\selectfont Breakout} game over other methods. 
We find that all methods have poor performance in the {\fontfamily{pcr}\selectfont Seaquest} game compared to human players (score of 100).
The potential explanation for the low reward in this game might be the complex rules of the game and the low quality of the 10\% dataset, both of which make the model only see the low-reward region.

\subsection{Different Target RTG Strategies (\qa{Q2})}

The second problem we want to investigate is the influence of target RTG during the inference stage. The results are reported in Table~\ref{tab:target-rtg}.
The easiest way to set the target RTG is using a fixed value, e.g., the max value of RTG in the dataset (named Max in Table~\ref{tab:target-rtg}).
This may cause a severe mismatch because lots of states correspond to low RTG. Therefore, DT sets an initial target RTG with the max value and gradually reduces it by subtracting the observed reward (named DT-Scheduler in Table~\ref{tab:target-rtg}). However, this scheduler cannot avoid the OOD problem, where the RTG is unreachable since $p(R|s) = 0$. As shown in the results, our method achieves better performance than using both max value and DT scheduler. The reason is that we select the target RTG according to the distribution $p(R|s)$, which generalizes well to different states $s$.

To further explore the selection of target RTG, we conduct an ablation study of the critical threshold $\delta$. We plot the results in Figure~\ref{fig:delta-influence} with Walker2D and Breakout environments. 
We can see that reducing the value of $\delta$ consistently improves the performance, which is in line with our design that a small $\delta$ corresponds to a high target RTG. region. However, setting $\lambda$ to a too-small value still has the risk of causing the OOD conditioning problem.

\begin{table}[t]
\centering
\caption{Normalized Score on Atari with the 10\% dataset. The results of our method are averaged over 5 random seeds.}
\small{
  \begin{tabular}{l|C{1.4cm}C{1.1cm}C{1.0cm}C{1.0cm}}
    \toprule
    \textbf{Method} & \textbf{Breakout}  & \textbf{Q*bert} & \textbf{Pong} & \textbf{Seaquest } \\
    \midrule
    BC        &  136.5           &  38.3            &  1.9            &  1.6   \\
    QR-DQN    &  496.7           & 52.1             &  119.3          & 14.5   \\
    REM       &  282.4           &  63.7            &  98.7           &  \textbf{19.2}   \\
    CQL       &  889.0           &  103.0           &  130.7          &  17.9   \\
    DT        &  293.6           &  60.1            &  113.0          &  7.2  \\
    \midrule
    Ours      &  \textbf{1239.2}\scriptsize{$\pm$104.2} &  \textbf{117.4}\scriptsize{$\pm$13.5}  &  \textbf{138.0}\scriptsize{$\pm$2.2} &  7.1\scriptsize{$\pm$3.1}   \\
    \bottomrule
  \end{tabular}
}
\vspace{-2mm}
\label{tab:atari-results}
\end{table}

\begin{table}[t]
\centering
\caption{Comparison between different inference strategies in Walker2D task.}
\small{
  \begin{tabular}{l|ccc}
    \toprule
    \textbf{Target RTG} & \textbf{Med-Reply}  & \textbf{Medium} & \textbf{Med-Expert} \\
    \midrule
    Max                  &  69.4  &  77.3  &  110.2  \\
    DT-Scheduler         &  72.2  &  78.1  &  109.6  \\
    Ours ($\delta=0.1$)  &  80.6  &  82.3  &  111.0  \\
    \bottomrule
  \end{tabular}
}
\vspace{-2mm}
\label{tab:target-rtg}
\end{table}

\subsection{Target RTG v.s. Observed RTG (\qa{Q3})}

After analyzing the design choice of the target RTG, we now look at the relationship between the target RTG and the observed RTG, which further reveals the behavior of our method.
The results in Figure~\ref{fig:target-observed} indicate that the observed RTG can generally match the same value of the target RTG.
We find that target RTG usually starts from a medium value ($250\sim 300$) rather than a high value used in DT. The reason is that the robot has not begun to move at the beginning, thus leading to a medium RTG.
As timestep increases, the target RTG increases to the high RTG region, meaning that the robot reaches the states corresponding to high RTG in the dataset.

At the end of trajectories, there are some mismatch cases where the observed RTG is lower or higher than the target RTG. 
One explanation for the case {\fontfamily{pcr}\selectfont observed RTG} $>$ {\fontfamily{pcr}\selectfont target RTG} is that the model $p(R|s)$ underestimates the value of RTG due to the sub-optimality of the {\fontfamily{pcr}\selectfont Medium-Replay} dataset. 
In contrast, {\fontfamily{pcr}\selectfont observed RTG} $<$ {\fontfamily{pcr}\selectfont target RTG} is usually caused by the maximum episode length, which compulsorily terminates a good state that should have had a high RTG.
This phenomenon happens when the quality of the dataset is high, for example, the Expert dataset in the left-top corner of Figure~\ref{fig:target-observed}.

\begin{figure}
    \centering
    \includegraphics[width=0.5\textwidth]{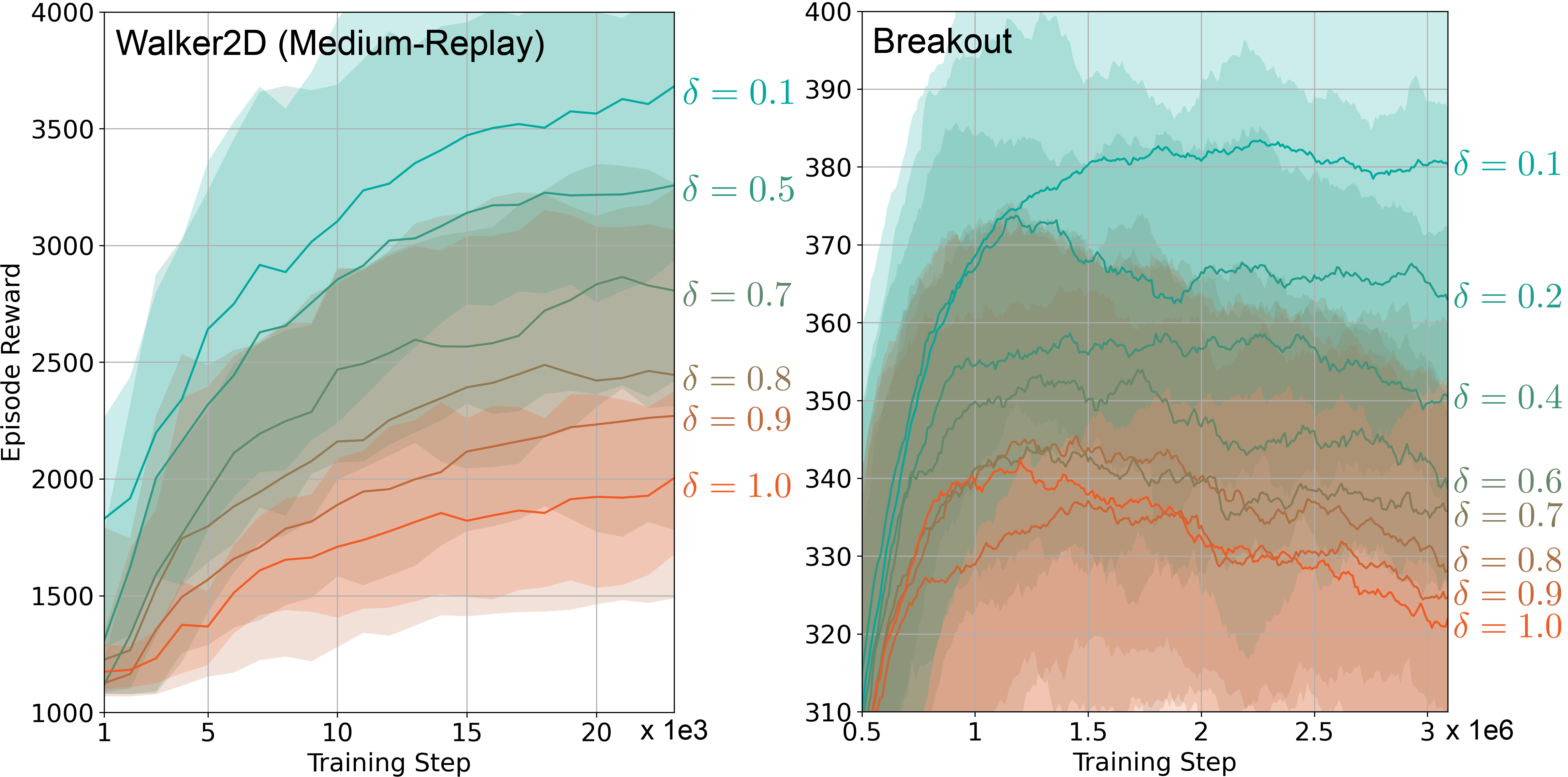}
    \caption{Raw episode reward of different values of $\delta$ during the inference stage in the Walker2D ({\fontfamily{pcr}\selectfont Medium-Replay}) task and the Breakout game. As the value of $\delta$ decreases, the performance improves.}
    \label{fig:delta-influence}
    \vspace{-3mm}
\end{figure}

\begin{figure}
    \centering
    \includegraphics[width=0.48\textwidth]{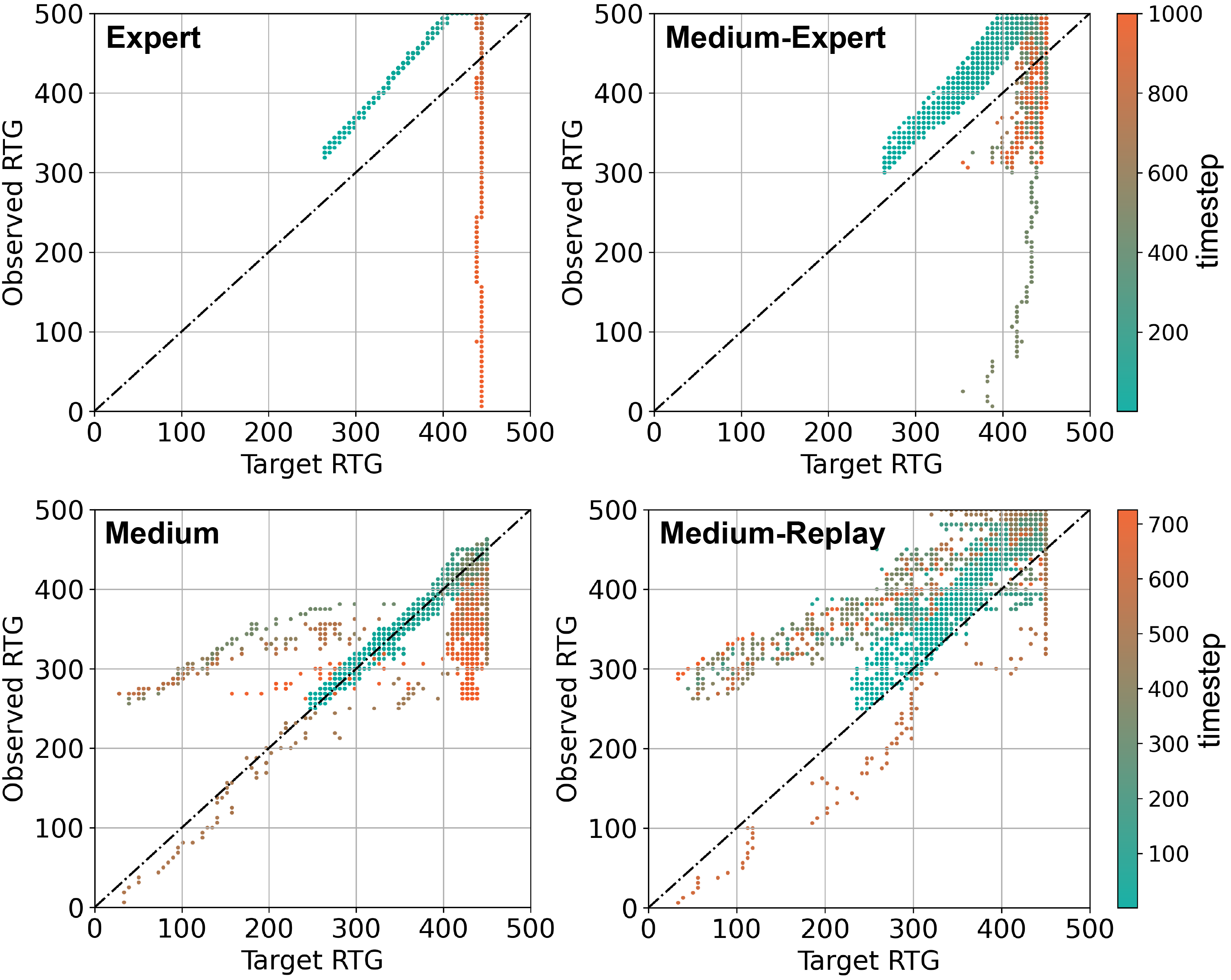}
    \caption{The relationship between target RTG and observed RTG in the Walker2D task with 4 datasets. Color represents the timestep in the trajectory.}
    \label{fig:target-observed}
    \vspace{-1mm}
\end{figure}

\subsection{Influence of Components (\qa{Q4})}

To study the contribution of each module in our model, we conduct ablation experiments by removing one component at one time and show the results in Table~\ref{tab:component}. 
We first test the model without using $\mathcal{L}_1$, which is important to solve the RTG independence problem. We find that removing this term causes a significant performance drop. We also observe that the model with $\mathcal{L}_1$ tends to ignore the RTG condition. 
We then modify the InfoNCE~\cite{oord2018representation} loss by using the uniform distribution instead of sampling from $\bar{\beta}_{\theta}(a|s)$. This modification also harms the performance since most negative samples from the uniform distribution are far from the valid action space, especially in high-dimensional action space.
Finally, we remove the Bayesian reparametrization (BR), which degenerates the model to a vanilla RCRL method.
We find that this variant achieves similar performance to the DT model.

\begin{table}[t]
\centering
\caption{Ablation study of different components in our method in the Walker2D task.}
\small{
  \begin{tabular}{l|ccc}
    \toprule
    \textbf{Model} & \textbf{Med-Reply}  & \textbf{Medium} & \textbf{Med-Expert} \\
    \midrule
    Full model               &  80.6  &  82.3  &  111.0  \\
    w/o $\mathcal{L}_{1}(\theta)$      &  72.3  &  79.5  &  108.9  \\
    w/o $a'\sim \bar{\beta}_{\theta}(a|s)$  &  77.1  &  80.6  &  110.3  \\
    w/o BR                   &  67.1  &  75.4  &  109.1  \\
    \bottomrule
  \end{tabular}
}
\label{tab:component}
\vspace{-3mm}
\end{table}

\section{Conclusion}

How to design appropriate inductive biases to improve generalization on high RTG inputs during training time and to avoid out-of-distribution RTG queries during the testing time are two core challenges in RCRL that were largely ignored by previous work.
This paper addresses these core challenges by proposing a novel set of inductive biases named Bayesian Reparameterized RCRL. 
Inspired by Bayes' theorem and causal relationships between random variables, our method successfully encodes the critical information that different RTG values are not independent classification problems but competitive. 
We also provide a causality perspective of our method to show that our parameterization of RCRL is in line with the ground-truth data generation process, which gains robustness to distribution shifts.
We demonstrate on standard offline benchmarks how our method significantly improved the generalization performance over previous methods. 
One potential limitation of our method is the additional computation introduced by the training and inference of the energy-based model compared to discriminative models used in Vanilla RCRL.


\bibliography{ref}

\begin{thebibliography}{59}
\providecommand{\natexlab}[1]{#1}
\providecommand{\url}[1]{\texttt{#1}}
\expandafter\ifx\csname urlstyle\endcsname\relax
  \providecommand{\doi}[1]{doi: #1}\else
  \providecommand{\doi}{doi: \begingroup \urlstyle{rm}\Url}\fi

\bibitem[Agarwal et~al.(2020)Agarwal, Schuurmans, and
  Norouzi]{agarwal2020optimistic}
Agarwal, R., Schuurmans, D., and Norouzi, M.
\newblock An optimistic perspective on offline reinforcement learning.
\newblock In \emph{International Conference on Machine Learning}, pp.\
  104--114. PMLR, 2020.

\bibitem[Ajay et~al.(2022)Ajay, Du, Gupta, Tenenbaum, Jaakkola, and
  Agrawal]{ajay2022conditional}
Ajay, A., Du, Y., Gupta, A., Tenenbaum, J., Jaakkola, T., and Agrawal, P.
\newblock Is conditional generative modeling all you need for decision-making?
\newblock \emph{arXiv preprint arXiv:2211.15657}, 2022.

\bibitem[Akkaya et~al.(2019)Akkaya, Andrychowicz, Chociej, Litwin, McGrew,
  Petron, Paino, Plappert, Powell, Ribas, et~al.]{akkaya2019solving}
Akkaya, I., Andrychowicz, M., Chociej, M., Litwin, M., McGrew, B., Petron, A.,
  Paino, A., Plappert, M., Powell, G., Ribas, R., et~al.
\newblock Solving rubik's cube with a robot hand.
\newblock \emph{arXiv preprint arXiv:1910.07113}, 2019.

\bibitem[Arjovsky et~al.(2019)Arjovsky, Bottou, Gulrajani, and
  Lopez-Paz]{arjovsky2019invariant}
Arjovsky, M., Bottou, L., Gulrajani, I., and Lopez-Paz, D.
\newblock Invariant risk minimization.
\newblock \emph{arXiv preprint arXiv:1907.02893}, 2019.

\bibitem[Bellemare et~al.(2017)Bellemare, Dabney, and
  Munos]{bellemare2017distributional}
Bellemare, M.~G., Dabney, W., and Munos, R.
\newblock A distributional perspective on reinforcement learning.
\newblock In \emph{International Conference on Machine Learning}, pp.\
  449--458. PMLR, 2017.

\bibitem[Boney et~al.(2020)Boney, Kannala, and Ilin]{boney2020regularizing}
Boney, R., Kannala, J., and Ilin, A.
\newblock Regularizing model-based planning with energy-based models.
\newblock In \emph{Conference on Robot Learning}, pp.\  182--191. PMLR, 2020.

\bibitem[Bridle(1989)]{bridle1989training}
Bridle, J.
\newblock Training stochastic model recognition algorithms as networks can lead
  to maximum mutual information estimation of parameters.
\newblock \emph{Advances in neural information processing systems}, 2, 1989.

\bibitem[Brown et~al.(2020)Brown, Mann, Ryder, Subbiah, Kaplan, Dhariwal,
  Neelakantan, Shyam, Sastry, Askell, et~al.]{brown2020language}
Brown, T., Mann, B., Ryder, N., Subbiah, M., Kaplan, J.~D., Dhariwal, P.,
  Neelakantan, A., Shyam, P., Sastry, G., Askell, A., et~al.
\newblock Language models are few-shot learners.
\newblock \emph{Advances in neural information processing systems},
  33:\penalty0 1877--1901, 2020.

\bibitem[Chen et~al.(2021)Chen, Lu, Rajeswaran, Lee, Grover, Laskin, Abbeel,
  Srinivas, and Mordatch]{chen2021decision}
Chen, L., Lu, K., Rajeswaran, A., Lee, K., Grover, A., Laskin, M., Abbeel, P.,
  Srinivas, A., and Mordatch, I.
\newblock Decision transformer: Reinforcement learning via sequence modeling.
\newblock \emph{Advances in neural information processing systems},
  34:\penalty0 15084--15097, 2021.

\bibitem[Dabney et~al.(2018)Dabney, Rowland, Bellemare, and
  Munos]{dabney2018distributional}
Dabney, W., Rowland, M., Bellemare, M., and Munos, R.
\newblock Distributional reinforcement learning with quantile regression.
\newblock In \emph{Proceedings of the AAAI Conference on Artificial
  Intelligence}, volume~32, 2018.

\bibitem[Ding et~al.(2022)Ding, Lin, Li, and Zhao]{ding2022generalizing}
Ding, W., Lin, H., Li, B., and Zhao, D.
\newblock Generalizing goal-conditioned reinforcement learning with variational
  causal reasoning.
\newblock \emph{arXiv preprint arXiv:2207.09081}, 2022.

\bibitem[Dosovitskiy et~al.(2020)Dosovitskiy, Beyer, Kolesnikov, Weissenborn,
  Zhai, Unterthiner, Dehghani, Minderer, Heigold, Gelly,
  et~al.]{dosovitskiy2020image}
Dosovitskiy, A., Beyer, L., Kolesnikov, A., Weissenborn, D., Zhai, X.,
  Unterthiner, T., Dehghani, M., Minderer, M., Heigold, G., Gelly, S., et~al.
\newblock An image is worth 16x16 words: Transformers for image recognition at
  scale.
\newblock \emph{arXiv preprint arXiv:2010.11929}, 2020.

\bibitem[Eberhardt \& Scheines(2007)Eberhardt and
  Scheines]{eberhardt2007interventions}
Eberhardt, F. and Scheines, R.
\newblock Interventions and causal inference.
\newblock \emph{Philosophy of science}, 74\penalty0 (5):\penalty0 981--995,
  2007.

\bibitem[Emmons et~al.(2021)Emmons, Eysenbach, Kostrikov, and
  Levine]{emmons2021rvs}
Emmons, S., Eysenbach, B., Kostrikov, I., and Levine, S.
\newblock Rvs: What is essential for offline rl via supervised learning?
\newblock \emph{arXiv preprint arXiv:2112.10751}, 2021.

\bibitem[Florence et~al.(2022)Florence, Lynch, Zeng, Ramirez, Wahid, Downs,
  Wong, Lee, Mordatch, and Tompson]{florence2022implicit}
Florence, P., Lynch, C., Zeng, A., Ramirez, O.~A., Wahid, A., Downs, L., Wong,
  A., Lee, J., Mordatch, I., and Tompson, J.
\newblock Implicit behavioral cloning.
\newblock In \emph{Conference on Robot Learning}, pp.\  158--168. PMLR, 2022.

\bibitem[Fu et~al.(2020)Fu, Kumar, Nachum, Tucker, and Levine]{fu2020d4rl}
Fu, J., Kumar, A., Nachum, O., Tucker, G., and Levine, S.
\newblock D4rl: Datasets for deep data-driven reinforcement learning.
\newblock \emph{arXiv preprint arXiv:2004.07219}, 2020.

\bibitem[Fujimoto \& Gu(2021)Fujimoto and Gu]{fujimoto2021minimalist}
Fujimoto, S. and Gu, S.~S.
\newblock A minimalist approach to offline reinforcement learning.
\newblock \emph{Advances in neural information processing systems},
  34:\penalty0 20132--20145, 2021.

\bibitem[Fujimoto et~al.(2019)Fujimoto, Meger, and Precup]{fujimoto2019off}
Fujimoto, S., Meger, D., and Precup, D.
\newblock Off-policy deep reinforcement learning without exploration.
\newblock In \emph{International conference on machine learning}, pp.\
  2052--2062. PMLR, 2019.

\bibitem[Ghazvininejad et~al.(2019)Ghazvininejad, Levy, Liu, and
  Zettlemoyer]{ghazvininejad2019mask}
Ghazvininejad, M., Levy, O., Liu, Y., and Zettlemoyer, L.
\newblock Mask-predict: Parallel decoding of conditional masked language
  models.
\newblock \emph{arXiv preprint arXiv:1904.09324}, 2019.

\bibitem[Gutmann \& Hyv{\"a}rinen(2010)Gutmann and
  Hyv{\"a}rinen]{gutmann2010noise}
Gutmann, M. and Hyv{\"a}rinen, A.
\newblock Noise-contrastive estimation: A new estimation principle for
  unnormalized statistical models.
\newblock In \emph{Proceedings of the thirteenth international conference on
  artificial intelligence and statistics}, pp.\  297--304. JMLR Workshop and
  Conference Proceedings, 2010.

\bibitem[Haarnoja et~al.(2017)Haarnoja, Tang, Abbeel, and
  Levine]{haarnoja2017reinforcement}
Haarnoja, T., Tang, H., Abbeel, P., and Levine, S.
\newblock Reinforcement learning with deep energy-based policies.
\newblock In \emph{International conference on machine learning}, pp.\
  1352--1361. PMLR, 2017.

\bibitem[Haarnoja et~al.(2018)Haarnoja, Zhou, Hartikainen, Tucker, Ha, Tan,
  Kumar, Zhu, Gupta, Abbeel, et~al.]{haarnoja2018soft}
Haarnoja, T., Zhou, A., Hartikainen, K., Tucker, G., Ha, S., Tan, J., Kumar,
  V., Zhu, H., Gupta, A., Abbeel, P., et~al.
\newblock Soft actor-critic algorithms and applications.
\newblock \emph{arXiv preprint arXiv:1812.05905}, 2018.

\bibitem[Hafner et~al.(2020)Hafner, Lillicrap, Norouzi, and
  Ba]{hafner2020mastering}
Hafner, D., Lillicrap, T., Norouzi, M., and Ba, J.
\newblock Mastering atari with discrete world models.
\newblock \emph{arXiv preprint arXiv:2010.02193}, 2020.

\bibitem[Janner et~al.(2021)Janner, Li, and Levine]{janner2021offline}
Janner, M., Li, Q., and Levine, S.
\newblock Offline reinforcement learning as one big sequence modeling problem.
\newblock \emph{Advances in neural information processing systems},
  34:\penalty0 1273--1286, 2021.

\bibitem[Kiran et~al.(2021)Kiran, Sobh, Talpaert, Mannion, Al~Sallab, Yogamani,
  and P{\'e}rez]{kiran2021deep}
Kiran, B.~R., Sobh, I., Talpaert, V., Mannion, P., Al~Sallab, A.~A., Yogamani,
  S., and P{\'e}rez, P.
\newblock Deep reinforcement learning for autonomous driving: A survey.
\newblock \emph{IEEE Transactions on Intelligent Transportation Systems}, 2021.

\bibitem[Kostrikov et~al.(2021)Kostrikov, Nair, and
  Levine]{kostrikov2021offline}
Kostrikov, I., Nair, A., and Levine, S.
\newblock Offline reinforcement learning with implicit q-learning.
\newblock \emph{arXiv preprint arXiv:2110.06169}, 2021.

\bibitem[Kumar et~al.(2019{\natexlab{a}})Kumar, Fu, Soh, Tucker, and
  Levine]{kumar2019stabilizing}
Kumar, A., Fu, J., Soh, M., Tucker, G., and Levine, S.
\newblock Stabilizing off-policy q-learning via bootstrapping error reduction.
\newblock \emph{Advances in Neural Information Processing Systems}, 32,
  2019{\natexlab{a}}.

\bibitem[Kumar et~al.(2019{\natexlab{b}})Kumar, Peng, and
  Levine]{kumar2019reward}
Kumar, A., Peng, X.~B., and Levine, S.
\newblock Reward-conditioned policies.
\newblock \emph{arXiv preprint arXiv:1912.13465}, 2019{\natexlab{b}}.

\bibitem[Kumar et~al.(2020)Kumar, Zhou, Tucker, and
  Levine]{kumar2020conservative}
Kumar, A., Zhou, A., Tucker, G., and Levine, S.
\newblock Conservative q-learning for offline reinforcement learning.
\newblock \emph{Advances in Neural Information Processing Systems},
  33:\penalty0 1179--1191, 2020.

\bibitem[LeCun et~al.(2006)LeCun, Chopra, Hadsell, Ranzato, and
  Huang]{lecun2006tutorial}
LeCun, Y., Chopra, S., Hadsell, R., Ranzato, M., and Huang, F.
\newblock A tutorial on energy-based learning.
\newblock \emph{Predicting structured data}, 1\penalty0 (0), 2006.

\bibitem[Levine et~al.(2020)Levine, Kumar, Tucker, and Fu]{levine2020offline}
Levine, S., Kumar, A., Tucker, G., and Fu, J.
\newblock Offline reinforcement learning: Tutorial, review, and perspectives on
  open problems.
\newblock \emph{arXiv preprint arXiv:2005.01643}, 2020.

\bibitem[Lillicrap et~al.(2015)Lillicrap, Hunt, Pritzel, Heess, Erez, Tassa,
  Silver, and Wierstra]{lillicrap2015continuous}
Lillicrap, T.~P., Hunt, J.~J., Pritzel, A., Heess, N., Erez, T., Tassa, Y.,
  Silver, D., and Wierstra, D.
\newblock Continuous control with deep reinforcement learning.
\newblock \emph{arXiv preprint arXiv:1509.02971}, 2015.

\bibitem[Liu et~al.(2020{\natexlab{a}})Liu, He, Xu, and Zhang]{liu2020energy}
Liu, M., He, T., Xu, M., and Zhang, W.
\newblock Energy-based imitation learning.
\newblock \emph{arXiv preprint arXiv:2004.09395}, 2020{\natexlab{a}}.

\bibitem[Liu et~al.(2020{\natexlab{b}})Liu, See, Ngiam, Celi, Sun, Feng,
  et~al.]{liu2020reinforcement}
Liu, S., See, K.~C., Ngiam, K.~Y., Celi, L.~A., Sun, X., Feng, M., et~al.
\newblock Reinforcement learning for clinical decision support in critical
  care: comprehensive review.
\newblock \emph{Journal of medical Internet research}, 22\penalty0
  (7):\penalty0 e18477, 2020{\natexlab{b}}.

\bibitem[Lu et~al.(2021)Lu, Wu, Hern{\'a}ndez-Lobato, and
  Sch{\"o}lkopf]{lu2021invariant}
Lu, C., Wu, Y., Hern{\'a}ndez-Lobato, J.~M., and Sch{\"o}lkopf, B.
\newblock Invariant causal representation learning for out-of-distribution
  generalization.
\newblock In \emph{International Conference on Learning Representations}, 2021.

\bibitem[Mnih et~al.(2013)Mnih, Kavukcuoglu, Silver, Graves, Antonoglou,
  Wierstra, and Riedmiller]{mnih2013playing}
Mnih, V., Kavukcuoglu, K., Silver, D., Graves, A., Antonoglou, I., Wierstra,
  D., and Riedmiller, M.
\newblock Playing atari with deep reinforcement learning.
\newblock \emph{arXiv preprint arXiv:1312.5602}, 2013.

\bibitem[Mnih et~al.(2015)Mnih, Kavukcuoglu, Silver, Rusu, Veness, Bellemare,
  Graves, Riedmiller, Fidjeland, Ostrovski, et~al.]{mnih2015human}
Mnih, V., Kavukcuoglu, K., Silver, D., Rusu, A.~A., Veness, J., Bellemare,
  M.~G., Graves, A., Riedmiller, M., Fidjeland, A.~K., Ostrovski, G., et~al.
\newblock Human-level control through deep reinforcement learning.
\newblock \emph{nature}, 518\penalty0 (7540):\penalty0 529--533, 2015.

\bibitem[Nachum \& Yang(2021)Nachum and Yang]{nachum2021provable}
Nachum, O. and Yang, M.
\newblock Provable representation learning for imitation with contrastive
  fourier features.
\newblock \emph{Advances in Neural Information Processing Systems},
  34:\penalty0 30100--30112, 2021.

\bibitem[Neal et~al.(2011)]{neal2011mcmc}
Neal, R.~M. et~al.
\newblock Mcmc using hamiltonian dynamics.
\newblock \emph{Handbook of markov chain monte carlo}, 2\penalty0
  (11):\penalty0 2, 2011.

\bibitem[Oord et~al.(2018)Oord, Li, and Vinyals]{oord2018representation}
Oord, A. v.~d., Li, Y., and Vinyals, O.
\newblock Representation learning with contrastive predictive coding.
\newblock \emph{arXiv preprint arXiv:1807.03748}, 2018.

\bibitem[Ouyang et~al.(2022)Ouyang, Wu, Jiang, Almeida, Wainwright, Mishkin,
  Zhang, Agarwal, Slama, Ray, et~al.]{ouyang2022training}
Ouyang, L., Wu, J., Jiang, X., Almeida, D., Wainwright, C.~L., Mishkin, P.,
  Zhang, C., Agarwal, S., Slama, K., Ray, A., et~al.
\newblock Training language models to follow instructions with human feedback.
\newblock \emph{arXiv preprint arXiv:2203.02155}, 2022.

\bibitem[Peters et~al.(2017)Peters, Janzing, and
  Sch{\"o}lkopf]{peters2017elements}
Peters, J., Janzing, D., and Sch{\"o}lkopf, B.
\newblock \emph{Elements of causal inference: foundations and learning
  algorithms}.
\newblock The MIT Press, 2017.

\bibitem[Pomerleau(1988)]{pomerleau1988alvinn}
Pomerleau, D.~A.
\newblock Alvinn: An autonomous land vehicle in a neural network.
\newblock \emph{Advances in neural information processing systems}, 1, 1988.

\bibitem[Prudencio et~al.(2022)Prudencio, Maximo, and
  Colombini]{prudencio2022survey}
Prudencio, R.~F., Maximo, M.~R., and Colombini, E.~L.
\newblock A survey on offline reinforcement learning: Taxonomy, review, and
  open problems.
\newblock \emph{arXiv preprint arXiv:2203.01387}, 2022.

\bibitem[Ramesh et~al.(2022)Ramesh, Dhariwal, Nichol, Chu, and
  Chen]{ramesh2022hierarchical}
Ramesh, A., Dhariwal, P., Nichol, A., Chu, C., and Chen, M.
\newblock Hierarchical text-conditional image generation with clip latents.
\newblock \emph{arXiv preprint arXiv:2204.06125}, 2022.

\bibitem[Sch{\"o}lkopf et~al.(2021)Sch{\"o}lkopf, Locatello, Bauer, Ke,
  Kalchbrenner, Goyal, and Bengio]{scholkopf2021toward}
Sch{\"o}lkopf, B., Locatello, F., Bauer, S., Ke, N.~R., Kalchbrenner, N.,
  Goyal, A., and Bengio, Y.
\newblock Toward causal representation learning.
\newblock \emph{Proceedings of the IEEE}, 109\penalty0 (5):\penalty0 612--634,
  2021.

\bibitem[Schulman et~al.(2015)Schulman, Levine, Abbeel, Jordan, and
  Moritz]{schulman2015trust}
Schulman, J., Levine, S., Abbeel, P., Jordan, M., and Moritz, P.
\newblock Trust region policy optimization.
\newblock In \emph{International conference on machine learning}, pp.\
  1889--1897. PMLR, 2015.

\bibitem[Schulman et~al.(2017)Schulman, Wolski, Dhariwal, Radford, and
  Klimov]{schulman2017proximal}
Schulman, J., Wolski, F., Dhariwal, P., Radford, A., and Klimov, O.
\newblock Proximal policy optimization algorithms.
\newblock \emph{arXiv preprint arXiv:1707.06347}, 2017.

\bibitem[Singh et~al.(2021)Singh, Kumar, and Singh]{singh2021reinforcement}
Singh, B., Kumar, R., and Singh, V.~P.
\newblock Reinforcement learning in robotic applications: a comprehensive
  survey.
\newblock \emph{Artificial Intelligence Review}, pp.\  1--46, 2021.

\bibitem[Song \& Ermon(2019)Song and Ermon]{song2019generative}
Song, Y. and Ermon, S.
\newblock Generative modeling by estimating gradients of the data distribution.
\newblock \emph{Advances in Neural Information Processing Systems}, 32, 2019.

\bibitem[Song \& Kingma(2021)Song and Kingma]{song2021train}
Song, Y. and Kingma, D.~P.
\newblock How to train your energy-based models.
\newblock \emph{arXiv preprint arXiv:2101.03288}, 2021.

\bibitem[Song et~al.(2020)Song, Sohl-Dickstein, Kingma, Kumar, Ermon, and
  Poole]{song2020score}
Song, Y., Sohl-Dickstein, J., Kingma, D.~P., Kumar, A., Ermon, S., and Poole,
  B.
\newblock Score-based generative modeling through stochastic differential
  equations.
\newblock \emph{arXiv preprint arXiv:2011.13456}, 2020.

\bibitem[Srivastava et~al.(2019)Srivastava, Shyam, Mutz, Ja{\'s}kowski, and
  Schmidhuber]{srivastava2019training}
Srivastava, R.~K., Shyam, P., Mutz, F., Ja{\'s}kowski, W., and Schmidhuber, J.
\newblock Training agents using upside-down reinforcement learning.
\newblock \emph{arXiv preprint arXiv:1912.02877}, 2019.

\bibitem[Sutton et~al.(1998)Sutton, Barto, et~al.]{sutton1998introduction}
Sutton, R.~S., Barto, A.~G., et~al.
\newblock Introduction to reinforcement learning.
\newblock 1998.

\bibitem[Vaswani et~al.(2017)Vaswani, Shazeer, Parmar, Uszkoreit, Jones, Gomez,
  Kaiser, and Polosukhin]{vaswani2017attention}
Vaswani, A., Shazeer, N., Parmar, N., Uszkoreit, J., Jones, L., Gomez, A.~N.,
  Kaiser, {\L}., and Polosukhin, I.
\newblock Attention is all you need.
\newblock \emph{Advances in neural information processing systems}, 30, 2017.

\bibitem[Welling \& Teh(2011)Welling and Teh]{welling2011bayesian}
Welling, M. and Teh, Y.~W.
\newblock Bayesian learning via stochastic gradient langevin dynamics.
\newblock In \emph{Proceedings of the 28th international conference on machine
  learning (ICML-11)}, pp.\  681--688, 2011.

\bibitem[Wu et~al.(2019)Wu, Tucker, and Nachum]{wu2019behavior}
Wu, Y., Tucker, G., and Nachum, O.
\newblock Behavior regularized offline reinforcement learning.
\newblock \emph{arXiv preprint arXiv:1911.11361}, 2019.

\bibitem[Yang et~al.(2022)Yang, Zhang, Song, Hong, Xu, Zhao, Shao, Zhang, Cui,
  and Yang]{yang2022diffusion}
Yang, L., Zhang, Z., Song, Y., Hong, S., Xu, R., Zhao, Y., Shao, Y., Zhang, W.,
  Cui, B., and Yang, M.-H.
\newblock Diffusion models: A comprehensive survey of methods and applications.
\newblock \emph{arXiv preprint arXiv:2209.00796}, 2022.

\bibitem[Yu et~al.(2021)Yu, Kumar, Rafailov, Rajeswaran, Levine, and
  Finn]{yu2021combo}
Yu, T., Kumar, A., Rafailov, R., Rajeswaran, A., Levine, S., and Finn, C.
\newblock Combo: Conservative offline model-based policy optimization.
\newblock \emph{Advances in neural information processing systems},
  34:\penalty0 28954--28967, 2021.

\end{thebibliography}
\bibliographystyle{icml2023}

\newpage
\appendix
\onecolumn

\section{Potential Negative Societal Impacts}

The main negative social impact of offline RL is that the learned policy purely relies on the dataset. Therefore, the policy could be subject to any bias in the dataset. Although our proposed method achieves strong out-of-distribution generalization, it may still be influenced by damaged data points. 
One way to mitigate this problem is to add a sanity check process before the training to ensure that the dataset is safe to use.

\section{Theoretical Proof}

\subsection{Proof of Theorem~\ref{the:2}}
\label{app:proof}

\begin{proof}
    For the first claim, we know that
    $\text{KL}(\pi^\delta(\mathcal{T})||\beta(\mathcal{T}))=\mathbb{E}_{\mathcal{T}\sim \pi^\beta}[\log \frac{\pi^\delta}{\beta}(\mathcal{T})]$. Then, we can get
\begin{align*}
\log\pi^\delta(\mathcal{T})/\beta(\mathcal{T}) =& \sum_{i=1}^N\log \pi^\delta(a_i|s_i)-\log\beta(a_i|s_i)\\
=& \sum_{i=1}^N\log\beta(Z\geq\theta_\delta|a_i,s_i)-\log \beta(Z\geq\theta_\delta|s_i)\\
\leq& -\sum_{i=1}^N \log\beta(Z\geq \theta_\delta|s_i) \\
\leq& -N\log \delta
\end{align*}
Combine these two formulas, we have $\text{KL}(\pi^\delta(\mathcal{T})||\beta(\mathcal{T}))\leq -N\log\delta$.

For the second claim, we first prove the following lemma: 
\begin{lemma}
    $X$ is a random variable, then for any $c\in\mathbb{R}$, we have $\mathbb{E}[X|X\geq c]\geq \mathbb{E}[X]$.
\end{lemma}
We know that
\begin{equation}
    \mathbb{E}[X|X\geq c]=\mathbb{E}_X[X\cdot \mathbbm{1}_{X\geq c}]/P(X\geq c).
\end{equation}
On the other hand, 
\begin{equation}
\begin{split}
    \mathbb{E}[X]\cdot P(X\geq c) 
    = & \mathbb{E}[X\cdot \mathbbm{1}_{X\geq c}]P(X\geq c)+\mathbb{E}[X\cdot \mathbbm{1}_{X< c}]P(X\geq c) \\
    = & \mathbb{E}[X\cdot \mathbbm{1}_{X\geq c}][1-P(X< c)]+\mathbb{E}[X\cdot \mathbbm{1}_{X< c}]P(X\geq c),
\end{split}
\end{equation}
where $\mathbbm{1}_{X\geq c}$ is an indicator function that outputs 1 when $X\geq c$ is satisfied. Thus we only need to show 
\begin{equation}
    \mathbb{E}[X\cdot \mathbbm{1}_{X< c}]/P(X< c)\leq \mathbb{E}[X\cdot \mathbbm{1}_{X\geq c}]/P(X\geq c). 
\end{equation}
This is obvious because
\begin{equation}
    \mathbb{E}[X\cdot \mathbbm{1}_{X< c}]\leq c\cdot P(X<c),
\end{equation}
\begin{equation}
    \mathbb{E}[X\cdot \mathbbm{1}_{X\geq c}]\geq c\cdot P(X\geq c).
\end{equation}
Therefore, for a state $s\in S$, we can see that 
\begin{equation}
\begin{split}
    V^\pi(s) 
    =& \mathbb{E}_{a\sim \pi,\pi}[Z^\pi(s,a)] \\
    \geq & \mathbb{E}_{a\sim \beta,\beta}[Z^\beta(s,a)|Z^\beta(s,a) \geq \theta_\delta(s,a)]\\
    \geq & V^\beta(s)
\end{split}
\end{equation}
\end{proof}

\section{Additional Experiment Results}

\subsection{RTG Independence Problem in DT}

The policy model tends to isolate the prediction problems $a=f_{R_i}(s), i=1,2\cdots$ according to different RTGs $R_i$, instead of learning a generalizable mapping from $R$ to action. 
In RCRL, this limitation can briefly be understood as the model can only learn from high RTG samples, the low RTG samples cannot help improve the model's performance in high RTG regions.
We confirm this tendency by conducting an additional experiment (results in Table~\ref{app:rtg_problem}). We fit three DT models using different variants of the medium-replay dataset. Top $x\%$ means we only select the top $x\%$ of trajectories, ordered by episode RTG. We observe that these models achieve similar performance, which indicates that the prediction conditioned on high RTG is independent of the training samples with low RTG. In addition, we use the same setting to test our method and find that removing the low RTG samples has a negative influence on the results.

\begin{table}[h]
\centering
\caption{Performance on different portions of the dataset.}
\small{
  \begin{tabular}{l|cccccc}
    \toprule
    \textbf{Dataset} & DT (top $100\%$)  & DT (top $50\%$)  & DT (top $20\%$)  & Ours (top $100\%$)  & Ours (top $50\%$)  & Ours (top $20\%$) \\
    \midrule
    Halfcheetch      & 36.0              & 37.2             & 36.5             & 42.3 & 40.5 & 38.1 \\
    Hopper           & 77.3              & 78.5             & 77.1             & 98.3 & 94.1 & 90.3 \\
    Walker2D         & 65.5              & 64.4             & 66.2             & 80.6 & 76.4 & 73.2 \\
    \bottomrule
  \end{tabular}
}
\label{app:rtg_problem}
\end{table}

\subsection{Sampling Bias Dominance Problem in DT}

Given state $s$, we assume $c(s)$ is the threshold for our target RTG, which is usually high. The dataset is sub-optimal, so $p_d(R>c(s)|s)<\epsilon$, where $\epsilon$ is a small number. Consider the training dataset $D=\{(s_i,a_i,R_i)\}_{i=1}^N$. As we show above, due to RTG independence, the model in fact only can learn from $D_c=\{(s,a,R)\in D| R>c(s)\}$. However $|D_c|<<|D|$. Then the trained model will be affected by the large sampling bias due to $|D_c|$ being small. 

\subsection{OOD Conditioning in DT}

\begin{wrapfigure}{r}{0.6\textwidth} 
    \vspace{-6mm}
    \centering 
    \includegraphics[width=0.6\textwidth]{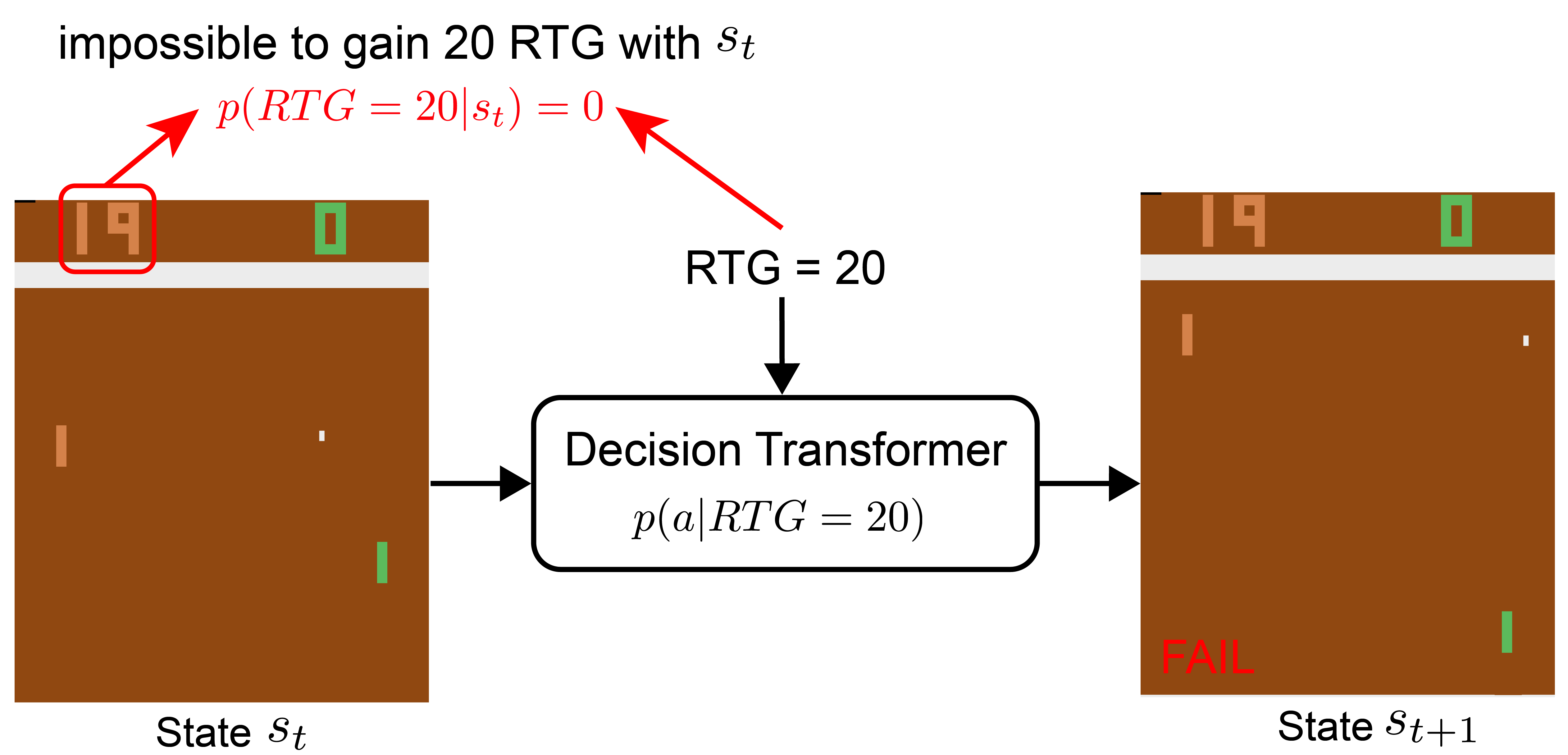}
    \vspace{-6mm}
    \caption{Example of the OOD condition problem.}
    \vspace{-4mm}
    \label{app:ood_condition}
\end{wrapfigure}

Given a state $s$, during training, we have the distribution of RTG $p_d(R|s)$ while during testing we aim to sample high RTG from $p_t(R|s)$. These two distributions can have non-overlap supports, namely $\text{KL}(p_t||p_d)$ usually can be $+\infty$. During test time, given a state $s$ and target RTG $R_t\sim p_t(R|s)$, the model try to predict $a\sim p(a|s, R=R_{t})$, however $R_t$ may be out-of-distribution (OOD) for $p_d$ in the sense that $p_d(R_t|s)=0$.
We provide a concrete example of the OOD target RTG in the Atari Pong game. The image is shown in Figure~\ref{app:ood_condition}. The game ends when one player gains 20 points. Since the RL agent already loses 19 points, it is almost impossible to obtain $RTG=20$ in this match. Therefore, $RTG=20$ is an OOD condition for the DT model.

\section{Experiment Details}

\subsection{Experiment Device}

The experiments were conducted on a device with 256GB memory and 2 $\times$ NVIDIA RTX A6000 GPUs. The Atari experiments require $\sim$150GB of memory to load the 10\% Atari dataset.

\subsection{Inference Optimizer of EBMs}

In Gym-Mujoco experiments, we use a derivative-free optimizer (DFO) proposed in~\cite{florence2022implicit} to infer the energy-based model. As stated in \cite{florence2022implicit}, other advantaged optimizers such as Langevin MCMC~\cite{welling2011bayesian} could improve the efficiency for high-dimensional cases.

We show the statistic of running time for inference of EBM in Table~\ref{app:inference}.
Although the inference spends more time than directly using conditional policy, the cost is still acceptable since we only need a few iterations.

\begin{table}[h]
\centering
\caption{Inference time of EBM.}
\small{
  \begin{tabular}{l|ccccccc}
    \toprule
    \textbf{Environment} & Halfcheetch &  Hopper  &  Walker2D  &  Breakout  &  Q*bert    &  Pong    &  Seaquest \\
    \midrule
    Inference time       &  0.0082 s   & 0.0080 s & 0.0082 s   &  0.0029 s  &  0.0031 s  & 0.0032 s &  0.0031 s \\
    \bottomrule
  \end{tabular}
}
\label{app:inference}
\end{table}

\subsection{Hyperparameters}

The hyperparameters used in Gym-Mujoco experiments and Atari experiments are summarized in Table~\ref{tab:hyper-d4rl} and Table~\ref{tab:hyper-atari}, respectively. We use the same hyperparameters for all experiments in the same benchmark. The source code of our experiments will be released after the blind review process.

\begin{table}[ht]
\centering
\small{
  \begin{tabular}{c|c|c}
    \toprule
    \textbf{Notation} & \textbf{Parameter Description} & \textbf{Value}  \\
    \midrule
                            & training iteration                         & 70,000 \\
                            & learning rate                              & 0.0005 \\
                            & batch size                                 & 512 \\
                            & action penalty                             & 0.0 \\
    $\lambda$               & weight of $\mathcal{L}_1(\theta)$          & 1.0 \\
    $|B|$                   & number of reward bucket                  & 80 \\
    $\gamma$                & reward discount                            & 0.99 \\
    $V_{min}$               & minimal bucket RTG                         & 0 \\
    $V_{max}$               & maximal bucket RTG                         & 1,200 \\
    $N_{A'}$                & number of negative samples during training & 256 \\
    \midrule
    $\delta$                & inference threshold                        &  0.1  \\
                            & number of episodes for each testing point  &  10  \\
    \midrule
                            & Number of iterative in DFO       & 5 \\
                            & Number of samples in DFO         & 65,536 \\
                            & Noise shrink parameter in DFO    & 0.9 \\
                            & Scale of noise in DFO            & 0.5 \\
    \bottomrule
  \end{tabular}
}
\caption{Hyperparameters for Gym-Mujoco experiments}
\label{tab:hyper-d4rl}
\end{table}

\begin{table}[ht]
\centering
\small{
  \begin{tabular}{c|c|c}
    \toprule
    \textbf{Notation} & \textbf{Parameter Description} & \textbf{Value}  \\
    \midrule
                            & training iteration                         & 3,000,000 \\
                            & learning rate                              & 0.00025 \\
                            & action penalty                             & 0.5 \\
                            & target network update frequency            & 8,000 \\
    $B$                     & batch size                                 & 32 \\
    $\lambda$               & weight of $\mathcal{L}_1(\theta)$          & 20.0 \\
    $|B|$                   & number of reward bucket                    & 51 \\
    $\gamma$                & reward discount                            & 0.95 \\
    $V_{min}$               & minimal bucket of RTG                      & 0 \\
    $V_{max}$               & maximal bucket RTG                         & 10 \\
    \midrule
    $\delta$                & inference threshold                         &  0.1  \\
    $\epsilon$              & exploration ratio during test               &  0.01  \\
                            & number of episodes for each testing point   &  10  \\
    \bottomrule
  \end{tabular}
}
\caption{Hyperparameters for Atari experiments}
\label{tab:hyper-atari}
\end{table}

\end{document}